\documentclass[a4paper]{article} 

\usepackage{cite}
\usepackage{enumerate}
\usepackage{amsmath}
\usepackage{amsfonts}
\usepackage{amssymb}
\usepackage{amsbsy}
\usepackage{theorem}
\usepackage{algorithm}
\usepackage{algpseudocode}
\usepackage{mathrsfs}
\usepackage{paralist}
\usepackage{psfrag}
\usepackage{subfig}
\usepackage{makeidx} 
\usepackage{pgf,pgfarrows,pgfautomata,pgfheaps,pgfnodes,pgfshade}

\theoremstyle{plain} 
\theoremstyle{plain} \newtheorem{definition}{Definition}[section]
\theoremstyle{plain} 
\theoremstyle{plain} 
\theoremstyle{plain} \newtheorem{proposition}{Proposition}[section]
\theoremstyle{plain} 
\theoremstyle{plain} 
\theoremstyle{plain} 
\theoremstyle{plain}

\newenvironment{proof}[1][Proof]{\begin{trivlist}
\item[\hskip \labelsep {\bfseries #1}]}{\end{trivlist}}
\newcommand{\qed}{\nobreak \ifvmode \relax \else
      \ifdim\lastskip<1.5em \hskip-\lastskip
      \hskip1.5em plus0em minus0.5em \fi \nobreak
      \vrule height0.5em width0.5em depth0.25em\fi}

\def\Reals {{\mathbb{R}}}

\def\CA {{\mathcal{A}}}

\def\CM {{\mathcal{M}}}

\def\CR {{\mathcal{R}}}
\def\CS {{\mathcal{S}}}
\def\CT {{\mathcal{T}}}

\def\argmax{\mathop{\rm arg\,max}}

\newcommand \E {\mathop{\mbox{\bf{E}}}}
\renewcommand \Pr {\mathop{\mbox{\bf{P}}}}
\newcommand \eq {{=}}
\newcommand \given {\mathrel{|}}

\newcommand \pl{{\pi}}
\newcommand \ps{{\pi^*}}

\newcommand \defn {\mathrel{\triangleq}}
\newcommand \MDPs {{\CM}}
\newcommand \BMDPs {{\MDPs_{B}}}

\newcommand \st {s_t}

\newcommand \at {a_t}
\newcommand \xit {\xi_t}

\newcommand \wt {\omega_t}
\newcommand \wti {\omega_t^i}

\newcommand \stn {s_{t+1}}
\newcommand \rtn {r_{t+1}}

\newcommand \xitn {\xi_{t+1}}
\newcommand \wtn {\omega_{t+1}}

\newcommand \stp {s_{t-1}}

\newcommand \atp {a_{t-1}}

\newcommand \wtki {\omega_{t+k}^i}

\newcommand \Vtps {{V_{t,T}^\ps}}
\newcommand \Vtnps {{V_{t+1,T}^\ps}}
\newcommand \Vtpm {{V_{t,T}^{\pl,\mu}}}
\newcommand \Vtnpm {{V_{\mu,t+1,T}^{\pl}}}

\newcommand \Beliefs {\Xi}
\newcommand \Actions {\CA}
\newcommand \States {\CS}

\newcommand \hats {\hat{a}_t^*}
\newcommand \has {\hat{a}^*}

\newcommand \bV {\bar{V}}

\newcommand \hv {\hat{v}}
\newcommand \hvs {\hv^*}

\newcommand \bv {\bar{v}}

\newcommand \tv {\tilde{v}}

\newcommand {\citep} {\cite}
\newcommand {\citet} {\cite}

\begin{document}

\author{Christos Dimitrakakis\\
  EPFL\\
  Lausanne\\
  Switzerland\\
  christos.dimitrakakis@gmail.com}

\title{Tree Exploration for Bayesian RL Exploration\footnote{This is a corrected and slightly expanded version of the homonymous paper presented at CIMCA'08.}}

\maketitle
\thispagestyle{empty}

\begin{abstract} 
  Research in reinforcement learning has produced algorithms for
  optimal decision making under uncertainty that fall within two main
  types.  The first employs a Bayesian framework, where optimality
  improves with increased computational time.  This is because the
  resulting planning task takes the form of a dynamic programming
  problem on a belief tree with an infinite number of states.  The
  second type employs relatively simple algorithm which are shown to
  suffer small regret within a distribution-free framework.  This
  paper presents a lower bound and a high probability upper bound on
  the optimal value function for the nodes in the Bayesian belief
  tree, which are analogous to similar bounds in POMDPs.  The bounds
  are then used to create more efficient strategies for exploring the
  tree.  The resulting algorithms are compared with the
  distribution-free algorithm UCB1, as well as a simpler baseline
  algorithm on multi-armed bandit problems.
\end{abstract}

\section{Introduction}

In recent
work~\citep{NIPS2007:ross:theoretical_heuristic_search_pomdp,Toussaint:probabilistic-POMDP:2006,NIPS2007:hoffman:bayesian_policy_learning,poupart2006asd,duff2002olc,NIPS2007:ross:bapomdp,duff97local},
Bayesian methods for exploration in Markov decision processes (MDPs)
and for solving known partially-observable Markov decision processes
(POMDPs), as well as for exploration in the latter case, have been
proposed. All such methods suffer from computational intractability
problems for most domains of interest.

The sources of intractability are two-fold.  Firstly, there may be no
compact representation of the current belief.  This is especially true
for POMDPs.  Secondly, optimally behaving under uncertainty requires
that we create an {\em augmented} MDP model in the form of a tree
\citep{duff2002olc}, where the root node is the current belief-state
pair and children are all possible subsequent belief-state pairs.
This tree grows large very fast, and it is particularly problematic to
grow in the case of continuous observations or actions.  In this work,
we concentrate on the second problem -- and consider algorithms for
expanding the tree.

Since the Bayesian exploration methods require a tree expansion to be
performed, we can view the whole problem as that of {\em nested}
exploration.  For the simplest exploration-exploitation trade-off
setting, bandit problems, there already exist nearly optimal,
computationally simple methods~\citep{JMLR:Auer:2002}.  Such methods
have recently been extended to tree
search~\citep{ECML:Kocsis+Szepesvari:2006}.  This work proposes to
take advantage of the special structure of belief trees in order to
design nearly-optimal algorithms for expansion of nodes.  In a sense,
by recognising that the tree expansion problem in Bayesian look-ahead
exploration methods is also an optimal exploration problem, we develop
tree algorithms that can solve this problem efficiently.  Furthermore,
we are able to derive interesting upper and lower bounds for the value
of branches and leaf nodes which can help limit the amount of search.
The ideas developed are tested in the multi-armed bandit setting for
which nearly-optimal algorithms already exist.

The remainder of this section introduces the augmented MDP formalism
employed within this work and discusses related
work. Section~\ref{sec:tree_expansion} discusses tree expansion in
exploration problems and introduces some useful bounds.  These bounds
are used in the algorithms detailed in Section~\ref{sec:algorithms},
which are then evaluated in Section~\ref{sec:experiments}.  We
conclude with an outlook to further developments.

\subsection{Preliminaries}
\label{sec:preliminaries}
We are interested in sequential decision problems where, at each time
step $t$, the agent seeks to maximise the expected utility
\[
\E[u_t \given \cdot] \defn \sum_{k=1}^{\infty} \gamma^k \E[r_{t+k} \given \cdot],
\]
where $r$ is a stochastic reward and $u_t$ is simply the discounted
sum of future rewards.  We shall assume that the sequence of rewards
arises from a Markov decision process, defined below.

\begin{definition}[Markov decision process]
  A Markov decision process (MDP) is defined as the tuple $\mu = (\CS,
  \CA, \CT, \CR)$ comprised of a set of states $\CS$, a set of actions
  $\CA$, a transition distribution $\CT$ conditioning the next state
  on the current state and action,
  \begin{equation}
    \CT(s'|s,a) \defn \mu(s_{t+1}\eq s'|s_t \eq s,a_t = a)
  \end{equation}
  satisfying the Markov property $\mu(s_{t+1} \given s_t, a_t) =
  \mu(\stn \given \st, \at, \stp, \atp, \ldots)$, and a
  reward distribution $\CR$ conditioned on states and actions:
  \begin{equation}
    \CR(r | s,a) \defn \mu(\rtn \eq r \given \st \eq s, \at \eq a),
  \end{equation}
  with $a \in \CA$, $s, s' \in \CS$, $r \in \Reals$.  Finally,
  \begin{equation}
    \mu(\rtn, \stn | \st, \at) = \mu(\rtn | \st,
    \at)\mu(\stn | \st, \at).
  \end{equation}
  \label{def:mdp}
\end{definition}
We shall denote the set of all MDPs as $\MDPs$.  For any policy $\pi$
that is an arbitrary distribution on actions, we can define a
$T$-horizon value function for an MDP $\mu \in \MDPs$ at time $t$ as:
\begin{align*}
  \label{eq:value_function}
  \Vtpm(s,a) &= 
  \E[\rtn \given \st \eq s, \at \eq a, \mu]\\
  &+ \gamma \sum_{s'} \mu(\stn \eq s' \given \st \eq s, \at \eq a) \Vtnpm(s').
\end{align*}
Note that for the infinite-horizon case, $\lim_{T \to \infty} \Vtpm =
V^{\pi, \mu}$ for all $t$.

In the case where the MDP is unknown, it is possible to use a Bayesian
framework to represent our uncertainty (c.f.~\citet{duff2002olc}).
This essentially works by maintaining a belief $\xit \in \Beliefs$,
about which MDP $\mu \in \MDPs$ corresponds to reality.  In a Bayesian
setting, $\xit(\mu)$ is our subjective probability measure that $\mu$
is true.

In order to optimally select actions in this framework, we need to use
the approach suggested originally in~\citep{bellman-kalaba:1959} under
the name of Adaptive Control Processes.  The approach was investigated
more fully in~\citep{duff97local,duff2002olc}.  This creates an {\em
  augmented} MDP, with a state comprised of the original MDP's state
$\st$ and our belief state $\xit$.  We can then solve the exploration
{\em in principle} via standard dynamic programming algorithms such as
backwards induction.  We shall call such models Belief-Augmented MDPs,
analogously to the Bayes-Adaptive MDPs of \cite{duff2002olc}.  This is
done by not only considering densities conditioned on the state-action
pairs $(s_t,a_t)$, i.e. $p(\rtn, \stn | \st, \at)$, but taking into
account the belief $\xit \in \Beliefs$, a probability space over
possible MDPs, i.e. augmenting the state space from $\CS$ to $\CS
\times \Beliefs$ and considering the following conditional density:
$p(r_{t+1}, s_{t+1}, \xitn \given s_t, a_t, \xit)$.  More formally, we
may give the following definition:
\begin{definition}[Belief-Augmented MDP]
  A {\em Belief-Augmented MDP} $\nu$ (BAMPD) is an MDP $\nu = (\Omega,
  \CA, \CT', \CR')$ where $\Omega = \CS \times \Beliefs$, where
  $\Beliefs$ is the set of probability measures on $\MDPs$, and $\CT',
  \CR'$ are the transition and reward distributions conditioned
  jointly on the MDP state $s_t$, the belief state $\xit$, and the
  action $a_t$.
  Here $\xit(\xitn | \rtn, \stn, \st, \at)$ is singular, so that we
  can define the transition
  \begin{align*}
    p(\wtn | \at, \wt)
    &\equiv
    p(\stn, \xitn | \at, \st, \xit).
  \end{align*}
\end{definition}
It should be obvious that $s_t, \xit$ jointly form a Markov state in
this setting, called the {\em hyper-state}.  In general, we shall
denote the components of a future hyper-state $\omega_t^i$ as $(s_t^i,
\xi_t^i)$.  However, in occassion we will abuse notation by referring
to the components of some hypserstate $\omega$ as $s_\omega,
\xi_\omega$.  We shall use $\BMDPs$ to denote the set of BMDPs.

As in the MDP case, finite horizon problems only require sampling all
future actions until the horizon $T$.
\begin{equation}
  \Vtps(\wt, \at) = \E[\rtn | \wt, \at]
  + \gamma  \int_\Omega \Vtnps(\wtn) \nu(\wtn|\wt, \at) \, d\wtn.
\end{equation}
However, because the set of hyper-states available at each time-step
is necessarily different from those at other time-steps, the value
function cannot be easily calculated for the infinite horizon case.

In fact, the only clear solution is to continue expanding a belief
tree until we are certain of the optimality of an action.  As has
previously been
observed~\citep{dearden98bayesian,dimitrakakis:icann:2006}, this is
possible since we can always obtain upper and lower bounds on the
utility of any policy from the current hyper-state.  We can apply
such bounds on future hyper-states in order to efficiently expand the
tree.

\subsection{Related work}
\label{sec:related_work}
Up to date, most work had only used full expansion of the belief tree
up to a certain depth.  A notable exception
is~\citet{Wang:bayesian-sparse-sampling:icml:2005}, which uses
Thompson sampling~\citep{thompson1933lou} to expand the tree.  In very
recent work~\citep{RossPineau:OnlinePlanningPOMDPs:jmlr2008}, the
importance of tree expansion in the closely related POMDP
setting\footnote{The BAMDP setting is equivalent to a POMDP where the
  unobservable part of the state is stationary, but continuous
  (chap. 5 \cite{duff2002olc})} has been recognised.  Therein, the
authors contrast and compare many different methods for tree
expansion, including branch-and-bound~\cite{mitten1970bab} methods and
Monte Carlo sampling.

Monte Carlo sampling methods have also been recently explored in the
upper confidence bounds on trees (UCT) algorithms, proposed
in~\citep{icml:gelly:uct:2007,ECML:Kocsis+Szepesvari:2006} in the
context of planning in games.  Our case is similar, however we can
take advantage of the special structure of the belief tree.  In
particular, for each node we can obtain high-probability upper and
lower bounds on the value of the optimal policy.

This paper's contribution is to recognise that tree expansion in
Bayesian exploration is itself an exploration problem with very
special properties.  Based on this insight, it proposes to combine
sampling with lower bounds and upper bound estimates at the leaves.
This allows us to obtain high-probability bounds for expansion of the
tree.  While the proposed methods are similar to the ones used in the
discrete-state POMDP
setting~\citep{RossPineau:OnlinePlanningPOMDPs:jmlr2008}, the BAMDP
requires the evaluation of different bounds at leaf nodes.  On the
experimental side, we present first results on bandit problems, for
which nearly-optimal distribution-free algorithms are known.  We
believe that this is a very important step towards extending the
applicability of Bayesian look-ahead methods in exploration.

\section{Belief tree expansion}
\label{sec:tree_expansion}
Let the current belief be $\xit$ and suppose we observe $x_t^i \defn
(s_{t+1}^i, r_{t+1}^i, a_t^i)$.  This observation defines a unique
subsequent belief $\xitn^i$.  Together with the MDP state $s$, this
creates a hyper-state transition from $\wt$ to $\wtn^i$. By
recursively obtaining observations for future beliefs, we can obtain
an unbalanced tree with nodes $\{\omega_{t+k}^i : k=1, \ldots, T; i=1,
\ldots\}$.  However, we cannot hope to be able to fully expand the
tree.  This is especially true in the case where observations
(i.e. states, rewards, or actions) are continuous, where we cannot
perform even a full single-step expansion.  Even in the discrete case
the problem is intractable for infinite horizons -- and far too
complex computationally for the finite horizon case.  However, had
there been efficient tree expansion methods, this problem would be
largely alleviated.  The remainder of this section details bounds and
algorithms that can be used to reduce the computational complexity of
the Bayesian lookahead approach.

\subsection{Expanding a given node}
\label{sec:node_expansion}
All tree search methods require the expansion of leaf nodes.  However,
in general, a leaf node may have an infinite number of children.  We
thus need some strategies to limit the number of children.

More formally, let us assume that we wish to expand in node $\wt^i =
(\xit^i, \st^i)$, with $\xit^i$ defining a density over $\MDPs$.  For
discrete state/action/reward spaces, we can simply enumerate all the
possible outcomes $\{\wtn^j\}_{j=1}^{|\CS \times \CA \times R|}$,
where $R$ is the set of possible reward outcomes.  Note that if the
reward is deterministic, there is only one possible outcome per
state-action pair.  The same holds if $\CT$ is deterministic, in both
cases making an enumeration possible.  While in general this may not
be the case, since rewards, states, or actions can be continuous, in
this paper we shall only examine the discrete case.

\subsection{Bounds on the optimal value function}
\label{sec:node_selection}
At each point in the process, the next node $\wtki$ to be expanded is
the one maximising a utility $U(\wtki)$.  Let $\Omega_T$ be the set of
leaf nodes.  If their values were known, then we could easily perform
the backwards induction procedure shown in
Algorithm~\ref{alg:backwards_induction}.
\begin{algorithm}[htb]
  \caption{Backwards induction action selection}
  \label{alg:backwards_induction}
  \begin{algorithmic}[1]
    \Procedure{BackwardsInduction}{$t, \nu, \Omega_T, V_T^*$}
    \For{$n=T-1, T-2, \ldots, t$}
    \For{$\omega \in \Omega_n$}
    \State $a(\omega) = \argmax_a [\E(r|\omega',\omega,\nu) +  \gamma V_{n+1}^* (\omega')] \nu(\omega'|\omega, a)$
    \State $V_n^*(\omega)\eq$ $\sum_{\omega' \in \Omega_{n+1}} \nu(\omega'|\omega,a^*_n) [\E(r|\omega',\omega,\nu) +  \gamma V_{n+1}^* (\omega')]$
    \EndFor
    \EndFor
    \State \Return $a^*_t$
    \EndProcedure
  \end{algorithmic}
\end{algorithm}
The main problem is obtaining a good estimate for $V^*_T$, i.e. the
value of leaf nodes. Let $\pi^*(\mu)$ denote the policy such that, for
any $\pi$,
\[
V^{\pi^*(\mu)}_\mu(s) \geq V^\pi_\mu(s) \qquad \forall s \in \States.
\]
Furthermore, let the maximum probability MDP arising from the belief
at hyper-state $\omega$ be $\hat{\mu}_\omega \defn \argmax_{\mu}
\hat{\mu}$.  Similarly, we denote the mean MDP with $\bar{\mu}_\omega
\defn \E[\mu | \xi_\omega]$.

\begin{proposition}
  The optimal value function at any leaf node $\omega$ is bounded by
  the following inequalities
  \begin{equation}
    \int V_\mu^{\pi^*(\mu)}(s_\omega) \xi_\omega(\mu) d\mu  \geq 
    V^*(\omega) \geq \int V_\mu^{\pi^*(\bar{\mu}_\omega)}(s_\omega) \xi_\omega(\mu) \, d\mu.
  \end{equation}
\end{proposition}
\begin{proof}
  By definition, $V^*(\omega) \geq V^\pi(\omega)$ for all $\omega$,
  for any policy $\pi$.  The lower bound follows trivially, since
  \begin{equation}
    \label{eq:value_lower_bound}
    V^{\pi^*(\bar{\mu}_\omega)}(\omega) \defn \int
    V_\mu^{\pi^*(\bar{\mu}_\omega)}(s_\omega) \xi_\omega(\mu) \, d\mu.
  \end{equation}
  The upper bound is derived as follows.  First note that for any
  function $f$, $\sup_x \int f(x, u) \,du \leq \int \sup_x f(x, u)
  \,du$.  Then, we remark that:
  \begin{subequations}
    \begin{align}
      V^*(\omega)
      &=
      \sup_\pi \int V_\mu^{\pi}(s_\omega) \xi_\omega(\mu) \, d\mu
      \label{eq:value_star_omega}
      \\
      &\leq
      \int \sup_\pi V_\mu^{\pi}(s_\omega) \xi_\omega(\mu) \, d\mu 
      \label{eq:expected_max_value}
      \\
      &=
      \int V_\mu^{\pi^*(\mu)}(s_\omega) \xi_\omega(\mu) \, d\mu.
      \label{eq:expected_max_policy_value}
    \end{align}
  \end{subequations}
  \qed
\end{proof}

In POMDPs, a trivial lower bound can be obtained by calculating the
value of the blind
policy~\citet{Hauskrecht:ValueFunctionApproximationPOMDP:jair,smith2005pbp},
which always takes the same action.  Our lower bound is in fact the
BAMDP analogue of the value of the blind policy in POMDPs.  This is
because for any {\em fixed policy} $\pi$, it holds trivially that
$V^\pi(\omega) \leq V^*(\omega)$.  In our case, we have made this
lower bound tighter by considering $\pi^*(\bar{\mu}_\omega)$, the
policy that is greedy with respect to the current mean estimate.

The upper bound itself is analogous to the POMDP value function bound
given in Theorem~9
of~\citet{Hauskrecht:ValueFunctionApproximationPOMDP:jair}.  However,
while the lower bound is easy to compute in our case, the upper bound
can only be approximated via Monte Carlo sampling with some
probability.

\subsection{Calculating the bounds}
In general, \eqref{eq:value_lower_bound} and
\eqref{eq:expected_max_value} cannot be expressed in closed form.
However, the integrals can be approximated via Monte Carlo sampling.
Let the leaf node which we wish to expand be $\omega$.  Then, we can
obtain $c$ MDP samples from the belief at $\omega$: $\mu_1, \ldots,
\mu_c \sim \xi_\omega(\mu)$.

The lower bound can be calculated by performing value iteration in the
mean MDP, in order to obtain the mean-MDP-optimal policy
$\pi^*(\bar{\mu}_\xi)$, where $\bar{\mu}_\xi$ is the {\em mean MDP}
for belief $\xi$.  If the beliefs $\xi$ can be expressed in closed
form, it is easy to calculate the mean transition distribution and the
mean reward from $\xi$.  For discrete state spaces, transitions can be
expressed as multinomial distributions, to which the Dirichlet density
is a conjugate prior.  In that case, for Dirichlet parameters
$\{\psi_{i}^{j,a}(\xi) : i, j \in \CS, a \in \CA\}$, we have
$\bar{\mu}_\xi(s'|s,a) = \psi_{s'}^{s,a}(\xi) / \sum_{i \in \CS}
\psi_{i}^{s,a} (\xi)$.  Similarly, for Bernoulli rewards, the
corresponding mean model arising from the beta prior with parameters
$\{\alpha^{s,a}(\xi), \beta^{s,a}(\xi) : s \in \CS, a \in \CA\}$ is
$\E[r|s, a, \bar{\mu}_\xi] = \alpha^{s,a}(\xi) / (\alpha^{s,a}(\xi) +
\beta^{s,a}(\xi))$.  Then optimal policy for the mean MDP can be found
with standard dynamic programming.

We can now use the mean-optimal polciy to obtain a stochastic lower
bound on the optimal value function. First, we calculate the value
function of the mean-optimal policy for each sampled MDP $\mu_k$. We
then average these samples to obtain the following approximation to
\eqref{eq:value_lower_bound}:
\begin{equation}
  \label{eq:empirical_lower_bound}
  \bar{v}_c(\omega) \defn \frac{1}{c}\sum_{k=1}^c V_{\mu_k}^{\pi^*(\bar{\mu}_\omega)}(s_\omega).
\end{equation}

For upper bounds, we follow a similar procedure. For each $\mu_k$, we
derive the optimal policy $\pi^*(\mu_k)$ and estimate its value
function $\tv^*_k \defn V^{\pi^*(\mu_k)}_{\mu} \equiv V^*_\mu$.  We
may then average these samples to obtain
\begin{equation}
  \label{eq:empirical_upper_bound}
  \hat{v}^*_c(\omega) \defn \frac{1}{c}\sum_{k=1}^c \tilde{v}^*_k(s_\omega).
\end{equation}
Let $\bar{v}^*(\omega) = \int_{\CM} \xi_\omega(\mu) V_\mu^*(s_\omega)
\, d\mu$. It holds that $\lim_{c \to \infty} [\hat{v}_c] =
\bar{v}^*(\omega)$ and that $\E[\hat{v}_c] = \bar{v}^*(\omega)$.  Due
to the latter, we can apply a Hoeffding inequality
\begin{equation}
  \label{eq:high_probability_upper_bound}
  \Pr\left(
    |\hat{v}^*_c(\omega) -  \bar{v}^*(\omega)| > \epsilon\right)
  < 2\exp\left(-\frac{2c\epsilon^2}{(V_{\max} - V_{\min})^{2}}
  \right),
\end{equation}
thus bounding the error within which we estimate the upper bound.  For
$r_t \in [0,1]$ and discount factor $\gamma$, note that $V_{\max} -
V_{\min} \leq 1/(1-\gamma)$. A similar inequality holds for the lower bound.

\subsection{Bounds on parent nodes}

We can obtain upper and lower bounds on the value of every action $a
\in \Actions$, at any part of the tree, by iterating over $\Omega_t$,
the set of possible outcomes following $\omega_t$:
\begin{align}
  \bv(\wt, a) &= \sum_{i=1}^{|\Omega_t|} \Pr(\wti \given \wt, a)
  \left[
    r^i_{t} + \gamma \bv(\wti)
  \right]
  \\
  \hvs(\wt, a) &= \sum_{i=1}^{|\Omega_t|} \Pr(\wti \given \wt, a)
  \left[
    r^i_{t} + \gamma \hvs(\wti)
  \right],
\end{align}
where the probabilities are implicitly conditional on the beliefs at
each $\wt$.  For every node, we can calculate an upper and lower bound
on the value of all actions.  Obviously, if at the root node $\wt$,
there exists some $\hats$ such that $\bv(\wt, \hats) \geq \bv^*(\wt,a)$
for all $a$, then $\hats$ is unambiguously the optimal action.

However, in general, there may be some other action, $a'$, whose upper
bound is higher than the lower bound of $\hats$.  In that case, we
should expand either one of the two trees.

It is easy to see that the upper and lower bounds at any node $\wt$
can be expressed as a function of the respective bounds at the leaf
nodes.  Let $B(\wt, a)$ be the set of all branches from $\wt$ when
action $a$ is taken.  For each branch $b \in B(\wt, a)$, let $\xit(b)$
be the probability of the branch from $\wt$ and $u_t^b$ be the
discounted cumulative reward along the branch.  Finally, let $L(\wt,
a)$ be the set of leaf nodes reachable from $\wt$ and $\omega_b$ be
the specific node reachable from branch $b$.  Then, upper or lower
bounds on the value function can simply be expressed as $v(\wt, a) =
\sum_{b \in B(\wt, a)} u_t^b + \gamma^{t_b} \xi_{\wt}(b) v(\omega_b)$.
This would allow us to use a heuristic for greedily minimising the
uncertainty at any branch.  However, the algorithms we shall consider
here will only employ evaluation of upper and lower bounds.

\section{Algorithms}
\label{sec:algorithms}
At each time step $t$, $N$ expansions are performed, starting from
state $\omega = omega_t$.  At the $n$-th expansion, a utility function
$U$ is evaluated for every node $\omega$ in the set of leaf nodes
$L_n$.  The main difference among the algorithms is the way
$U$ is calculated.
\begin{algorithm}
  \caption{General tree expansion algorithm}
  \label{al:tree_expansion}
  \begin{algorithmic}[1]
    \Procedure{ExpandTree}{$\omega$, $N$}
    \State $n = 1$, $\omega^1 = \omega$.
    \State $L = \{\omega^1\}$
    \While{$n \leq N$}
    \State Take one more sample for $\hat{v}_c(\omega)$ from all leaf nodes $L_n$
    \State Calculate $U(\omega_i)$ for all $i \in L_n$.
    \State Calculate $C(\omega_b)$, the children of branch $b = \argmax_i U(\omega_i)$.
    \State $L_{n+1} = L_n \cup (L_{n+1} \backslash \omega_b)$.
    \State $n++$
    \EndWhile
    \State \Return a
    \EndProcedure
  \end{algorithmic}
\end{algorithm}

\begin{enumerate}
\item {\em Serial}.  This results in a nearly balanced tree, as the
  oldest leaf node is expanded next, i.e.  $U(\omega_i) = - i$,
  the negative node index.
\item {\em Random}.  In this case, we expand any of the current leaf
  nodes with equal probability, i.e. $U(\omega_i) = U(\omega_j)$ for
  all $i,j$.  This can of course lead to unbalanced trees.
\item {\em Highest lower bound}.  We expand the node maximising a lower bound
  i.e. $U(\omega_i) = \gamma^{t_i} \bV(\omega_i)$.
\item {\em Thompson sampling}.  We expand the node for which the
  currently sampled upper bound is highest, i.e. $U(\omega_i) = \gamma^{t_i}
  \tilde{v}(\omega_i)$.
\item {\em High probability upper bound}.  We expand the node with the
  highest mean upper bound $U(\omega_i) = \gamma^{t_i}
  \max_k\{\hat{v}^*_{c(i)}(\omega_i), \bar{v}(\omega_i)\}$.
\end{enumerate}
While methods 3 and 4 only use one sample from the upper bound
calculation at every iteration.  The last two methods retain the
samples obtained in the previous iterations and use them to calculate
the mean estimate.

\section{Experiments}
\label{sec:experiments}

We compared the regret of the tree expansion to the optimal policy in
bandit problems with Bernoulli rewards with two benchmarks: the UCB1
algorithm~\citep{mach:Auer+Cesa+Fischer:2002}, which suffers only
logarithmic regret, and secondly a Bayesian algorithm that is greedy
with respect to a mean Bayesian estimate with a prior density
$\textrm{Beta}(0,1)$, i.e. Algorithm~\ref{alg:bayesian_greedy} applied
with $\omega = (0, \xi_t)$, which is a simple optimistic heuristic for
such problems.
\begin{algorithm}[htb]
  \caption{Greedy Bayesian action selection}
  \label{alg:bayesian_greedy}
  \begin{algorithmic}[1]
    \Procedure{BayesianGreedy}{$\omega, \gamma$}
    \State Select  $\hat{a}^* = \argmax V^*_{\bar{\mu}}(s_\omega, a)$
    \EndProcedure
  \end{algorithmic}
\end{algorithm}

We compared the algorithms using the notion of expected undiscounted
regret accumulated over $T$ time steps, i.e. the expected loss that a
specific policy $\pi$ suffers over the policy which always chooses the
arm $a^*$ with the highest mean reward:
\[
\sum_{T=1}^T \E[r | a_t = a^*] - 
\sum_{T=1}^T \E[r | \pi].
\]
In order to determine the expected regret experimentally we must
perform multiple independent runs and average over them.

Figure~\ref{fig:bandit_results} shows the cumulative undiscounted
regret for horizon $T = 2/(1-\gamma)$, with $\gamma=0.9999$ and
$|\CA|=2$, averaged over 1000 runs.  We compare the UCB1 algorithm
({\em ucb}), and the Bayesian baseline ({\em base}) with the BAMDP
approach.  The figure shows the cumulative undiscounted regret as a
function of the number of look-aheads, for the following expansion
algorithms: {\em serial}, {\em random}, highest lower bound ({\em
  lower bound}), and high probability upper bound({\em upper bound}).
The last two algorithms use $\gamma$-rate discounting for future node
expansion.

It is evident from these results that the highest lower bound method
never improves beyond the first expansion.  This is due to the fact
that the lower bounds never change after the first step when this
algorithm is used.  The simple serial expansion seems to perform only
slightly better.  On the other hand, while the serial expansion is
consistently better than the random expansion, it does not manage to
achieve less than half of the regret of the latter.  It thus appears
as though the stochastic selection of branches is in itself of quite
some importance in this type of problem.  For problems with more arms
and longer horizons, the differences between methods are amplified.
The results are in agreement with those obtained in the POMDP setting,
where upper bound expansions appear to be
best~\citep{RossPineau:OnlinePlanningPOMDPs:jmlr2008}.  
\begin{figure}
  \centering
  \subfloat[$|\CA|=2, \gamma=0.999$]{
    \includegraphics[width=0.49\textwidth]{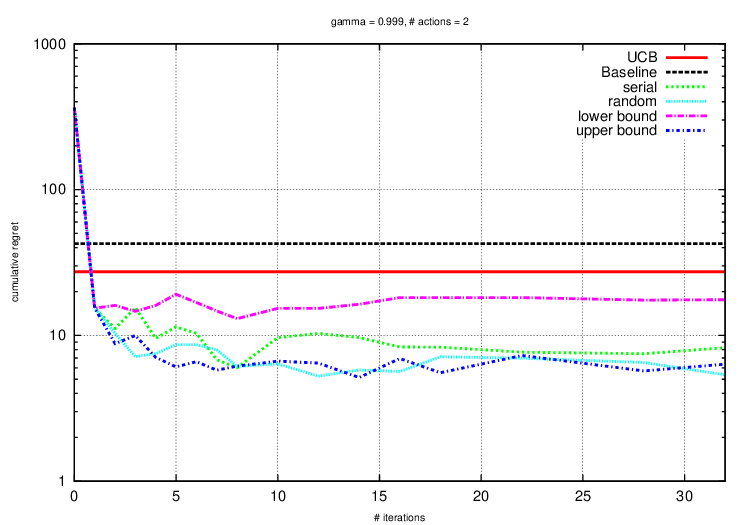}
  }
  \subfloat[$|\CA|=4, \gamma=0.999$]{
    \includegraphics[width=0.49\textwidth]{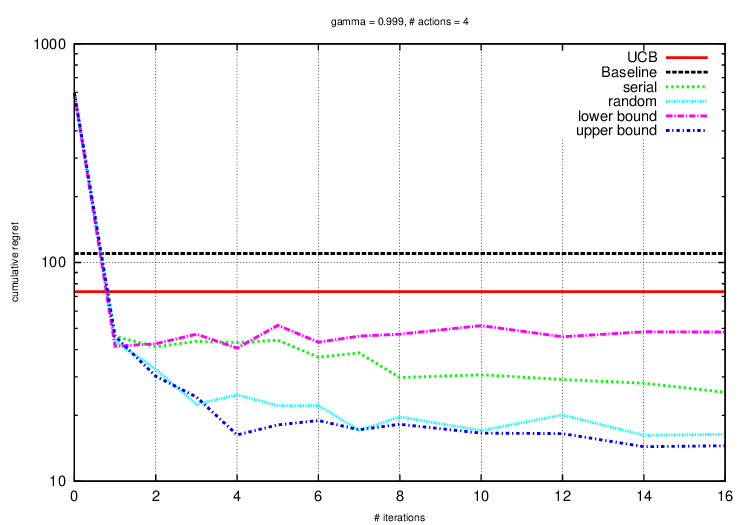}
  }
  \\
  \subfloat[$|\CA|=2, \gamma=0.9999$]{
    \includegraphics[width=0.49\textwidth]{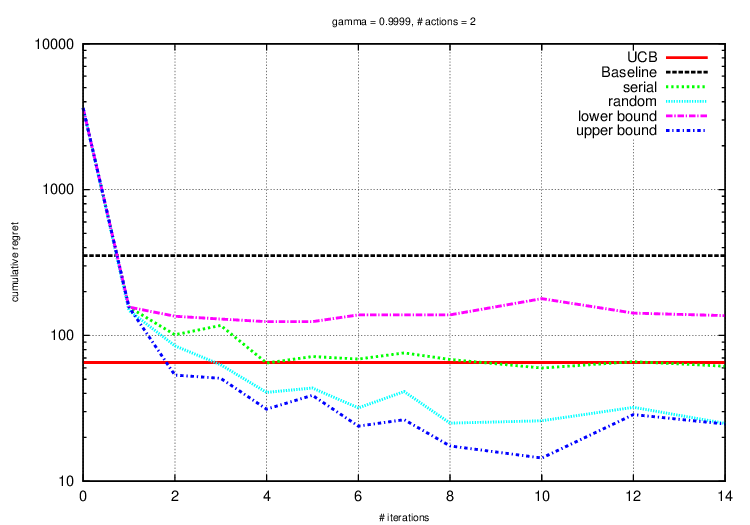}
  }
  \subfloat[$|\CA|=4, \gamma=0.9999$]{
    \includegraphics[width=0.49\textwidth]{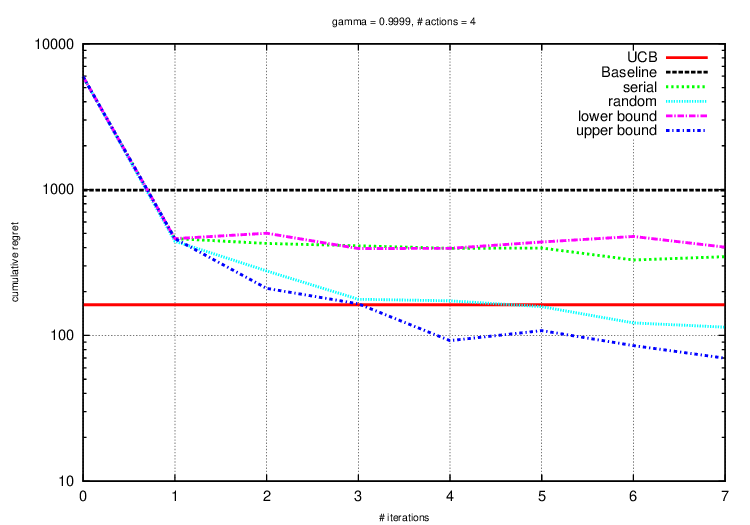}
  }
  \caption{Cumulative undiscounted regret accrued over $2/(1-\gamma)$ time-steps.}
  \label{fig:bandit_results}
\end{figure}
\nopagebreak
\section{Conclusion}
One of this paper's aims was to draw attention to the interesting
problem of tree expansion in Bayesian RL exploration.  To this end,
bounds on the optimal value function at belief tree leaf nodes have
been derived and then utilised as heuristics for tree expansion.  It
is shown experimentally that the resulting expansion methods have very
significant differences in computational complexity for bandit
problems.  While the results are preliminary in the sense that no
experiments on more complex problems are presented and that only very
simple expansion algorithms have been tried, they are nevertheless
significant in the sense that the effect of the tree exploration
method used is very large.

Apart from performing further experiments, especially with more
sophisticated expansion algorithms, future work should include
deriving bounds on the minimum and maximum depth reached for each
algorithm, as well as more general regret bounds if possible.  The
regret could be measured either as simply the $\epsilon, \delta$
optimality of $\has$, or, more interestingly, bounds on the cumulative
online regret suffered by each algorithm. More importantly, problems
with infinite observation spaces (i.e. with continuous rewards) should
also be examined.

My current work includes the analysis of the {\em stochastic}
branch-and-bound algorithm such as the ones described
in~\citep{Norkin:StochasticBnB:MP96,kleywegt:saa-stochastic-optimization:2001}.
This algorithm is essentially the same as the high probability upper
bound method used in the current paper.  Another interesting approach
would be to develop a new expansion algorithm that achieves a small
anytime regret, perhaps in the lines of
UCT~\citep{ECML:Kocsis+Szepesvari:2006}.  Such algorithms have been
very successful in solving problems with large spaces and may be
useful in this problem as well, especially when the space of
observations becomes larger.

\subsubsection*{Acknowledgments}
This work was supported by the ICIS-IAS project.  Thanks to Carsten
Cibura and Frans Groen for useful discussions and to Aikaterini
Mitrokotsa for proofreading.

\appendix
\section{The Bayesian inference in detail}

Let $\MDPs$ be the set of MPDs with unknown transition probabilities
and state space $\CS$ of size $n$. 
We denote our belief at time $t+1$ about which MDP is true as
\begin{subequations}
  \begin{align}
    \xitn(\mu) 
    &\defn \xit(\mu | s_{t+1}, s_t a_t)
    \\
    &=
    \frac{\mu(s_{t+1}|s_t, a_t,\mu) \xit(\mu)}
    {\int_{\MDPs}\mu'(s_{t+1}|s_t a_t,\mu') \xit(\mu') \, d\mu'}
  \end{align}
\end{subequations}

Since this is an infinite set of MDPs, we can have each MDP $\mu$
correspond to a particular probability distribution over the
state-action pairs.  More specifically, let us define for each state
action pair $s,a$, a Dirichlet distribution
\begin{equation}
  \label{eq:dirichlet}
  \xit(q_{s,a} = x)
  = \frac{\Gamma(\psi^{s,a}(t))}{\prod_{i \in \CS} \Gamma(\psi^{s,a}_i(t))}
  \prod_{i \in \CS} x_i^{\psi^{s,a}_i(t)}.
\end{equation}
with $q_{s,a} \defn \Pr(s_{t+1} \given s_t \eq s, a_t \eq a)$.
We will denote by $\Psi(t)$ the matrix of state-action-state
transition counts at time t, with $\Psi(0)$ being the matrix defining
our prior Dirichlet distriburtion.

We shall now model the joint prior over transition distributions as
simply the product of priors. Then we can denote the matrix of
state-action-state transition {\em probabilities} for MDP $\mu$ as
$Q^\mu$ and let $q^\mu_{s,a,i} \defn \mu(s_{t+1} \eq i \given s_t \eq
s, a_t \eq a)$.  Then
\begin{subequations}
  \begin{align}
    \xit(\mu)
    &= \xit(Q^\mu) = \xit(q_{s,a} = q_{s,a}^\mu \forall s \in \CS, a \in \CA)
    \\
    &= \prod_{s \in \CS}\prod_{a \in \CA} \xit(q_{s,a} = q_{s,a}^\mu),
    \\
    &= \prod_{s \in \CS}\prod_{a \in \CA} \frac{\Gamma(\psi^{s,a}(t))}{\prod_{i \in \CS} \Gamma(\psi^{s,a}_i(t))}
    \prod_{i \in \CS} \left(q^\mu_{s,a,i}\right)^{\psi^{s,a}_i(t)}.
  \end{align}
\end{subequations}
where we assume that each state-action pair's transition distribution
is independent of the other transition distributions.  This means that
$\Psi$ is a sufficient statistic for expressing the density over
$\MDPs$.

We can additionally model $\mu(\rtn | \st, \at)$ with a suitable
belief and assume independence.  This in no way complicates the
exposition for MDPs.

\bibliographystyle{plain}
\bibliography{../../../bib/misc,../../../bib/mine}

\end{document}